\documentclass{article}
\usepackage{spconf,amsmath,graphicx,hyperref}
\usepackage{amsfonts}
\usepackage{algorithmic}
\usepackage{array}
\usepackage{textcomp}
\usepackage{stfloats}
\usepackage{url}
\usepackage{verbatim}
\usepackage{graphicx}
\usepackage{balance}
\usepackage{cite}
\DeclareGraphicsExtensions{.JPG,.eps,.pdf}
\usepackage{amsthm}
\newtheorem*{remark}{Remark}

\usepackage{epstopdf}
\usepackage{color}
\usepackage{amscd}
\usepackage{soul}
\usepackage{stfloats}
\usepackage{atbegshi}
\usepackage{flushend}
\usepackage{mathtools}
\DeclarePairedDelimiter\norm{\lVert}{\rVert}
\usepackage{tikz,pgfplots,xcolor}

\usepackage{dsfont}
\usepackage{bbm}

\usepackage{amssymb}
\usepackage{float}

\usepackage[ruled,norelsize]{algorithm2e}
\usepackage{bm}
\usepackage{soul}
\newcommand{\Ex}{\mathbb{E}}
\newcommand{\bmtheta}{\bm{\theta}}

\newcommand{\vecx}{\mathbf{x}}
\newcommand{\vecy}{\mathbf{y}}

\newcommand{\vecw}{\mathbf{w}}

\newcommand{\vecs}{\mathbf{s}}

\usepackage{caption}
\usepackage{xcolor}
\usepackage{textcomp}

\newtheorem{theorem}{Theorem}

\newtheorem{lemma}{Lemma}

\title{Channel-Adaptive Robust Aggregation for Over-the-Air Federated Learning in Heterogeneous Networks}
\name{Zubaida Fatima$^{\dagger *}$,~ Zubair Shaban$^{\dagger *}$,~ Yusuf Jamal$^{\dagger}$,~ Nazreen Shah$^{\dagger}$,~ Ranjitha Prasad$^{\dagger}$,~ B. N. Bharath$^{\#}$}
  
  \address{$^{\dagger}$IIIT, Delhi, $^{\#}$IIT Dharwad}

\begin{document}
%\ninept
%
\maketitle
\footnotetext[1]{* These authors contributed equally to this work.}

\begin{abstract}
The growing demand for privacy-preserving, data-intensive applications such as IoT, augmented reality, and autonomous systems positions Federated Learning (FL) as a key enabler in 6G networks. Over-the-Air FL (OTA-FL) leverages the superposition property of the wireless multiple access channel for efficient aggregation via simultaneous transmissions. Existing methods rely on fixed aggregation schedules and do not jointly address noise, fading, and client heterogeneity. We propose CHARGE-FL (CHannel-Adaptive Robust agGrEgation), a framework that adaptively schedules aggregation based on channel dynamics and application readiness. By combining a tailored optimization strategy with a dual-purpose precoding mechanism, CHARGE-FL mitigates channel distortion and bias from partial updates, achieving superior accuracy, stability, and convergence under realistic wireless conditions. Empirical results under realistic wireless conditions show that CHARGE-FL significantly improves accuracy, stability, and convergence over state-of-the-art OTA-FL methods, particularly in straggler-prone and noisy scenarios.
\end{abstract}
\section{Introduction}
\label{sec:intro}

To support increasingly complex edge intelligence applications, 6G is envisioned as a key enabler of next-generation distributed systems by integrating ultra-reliable low-latency communication, high-speed connectivity, and AI-driven decentralized processing \cite{6GIOT}. Federated Learning (FL) plays a central role in this vision, allowing edge devices to collaboratively train global models while preserving privacy and reducing communication costs \cite{mcmahan2017communication,li2019convergence}. In FL, a central server coordinates training by broadcasting a global model, which clients update locally and return for aggregation. However, wireless FL faces critical challenges from limited uplink bandwidth, shared medium constraints, and client heterogeneity \cite{10278452}. Conventional orthogonal access (e.g., TDMA, FDMA) scales poorly with many clients, causing inefficiency and slow convergence. Over-the-Air FL (OTA-FL) addresses this by exploiting the superposition property of the wireless multiple access channels (MAC) for simultaneous analog aggregation \cite{yang2020federated,sery2021over}, but remains highly sensitive to noise, fading, and heterogeneity, especially under straggler effects where slow clients delay or drop out of the process \cite{ozfatura2020straggler,10625955,10890519,shaban2025noise}.
\newline
\noindent Many real-world applications demand aggregation that adapts not only to computational readiness and task urgency but also to rapidly varying wireless channels. In settings such as vehicular networks, smart manufacturing, or healthcare, updates may be triggered by events (e.g., obstacle detection, system alerts) and must be transmitted reliably over fluctuating channels. Similar requirements arise in autonomous driving and industrial IoT, where responsiveness and security are critical. However, existing OTA-FL frameworks \cite{sery2020analog,mao2022charles,reviewerref} typically enforce fixed local update counts before aggregation, reducing adaptability to heterogeneous clients. In distributed optimization, FedProx \cite{li2020federated} and FedNova \cite{wang2020tackling} address client heterogeneity, while SCAFFOLD \cite{karimireddy2020scaffold} reduces the impact of data heterogeneity, but these methods are not tailored for wireless environments. To improve robustness in such settings, CHARLES \cite{mao2022charles,yang2022over} employs adaptive scaling based on channel state information (CSI), while COTAF \cite{sery2021over} mitigates the effect of channel noise. Despite these advances, both assume synchronous communication rounds and lack the flexibility to adapt aggregation timing to channel dynamics or application-driven triggers. By contrast, channel adaptive and application-driven aggregation enables updates to be transmitted when channels are favorable or when external events demand timely action.
%Despite these advances, both assume synchronous communication rounds and lack the flexibility to enable channel adaptive and application-driven aggregation, which in contrast allows updates to be transmitted when channels are favorable or when external events demand timely action.
% \newline

\noindent \textbf{Contributions:} We propose CHARGE-FL (CHannel-Adaptive Robust agGrEgation), a flexible framework for adaptive, event- and channel-aware aggregation in federated learning over noisy wireless MAC. Unlike schemes that employ rigid schedules, CHARGE-FL dynamically triggers aggregation based on client readiness, application demands, or favorable channel conditions. It introduces a joint optimization with dual-purpose epoch-specific precoding to mitigate channel distortion and heterogeneity from partial client participation, improving robustness, convergence, and model fidelity. We provide a theoretical analysis showing that the optimality gap decreases at an $O(1/T)$ rate, making CHARGE-FL the first framework to jointly address adaptivity, wireless impairments, and heterogeneity with theoretical guarantees.

\section{System Model}
\label{sec:sysmodel}
We consider a federated learning setup with $N$ clients and a central server communicating over a shared wireless MAC. Each client $k$ holds a local dataset $\mathcal{D}_k$ of size $D_k$ with samples $\{\vecs_i^k\}_{i=1}^{D_k}$. The objective is to minimize the empirical loss
\begin{equation}
f_k(\bmtheta) = \tfrac{1}{D_k}\sum_{i=1}^{D_k} \ell(\vecs_i^k;\bmtheta),
\end{equation}
where $\ell(\cdot;\cdot)$ is a user-defined loss and $\bmtheta \in \mathbb{R}^d$ are the model parameters. Local training is performed via SGD:
\begin{equation}
\bmtheta_{t +1}^k = \bmtheta_t^k - \eta_t \nabla f_{i_t^k}(\bmtheta_t^k),
\label{eq:localSGD}
\end{equation}
where $\eta_t$ represents the step size and $\nabla f_{i_t^k}(\bmtheta_t^k)$ represents the stochastic gradient at sample $i_t^k$.

\noindent Unlike traditional OTA-FL frameworks that enforce a fixed number of local updates per round, we adopt a \textit{channel- and application-adaptive formulation} where aggregation is triggered by factors such as client readiness, learning progress, or application events. Conventional methods, such as COTAF, adopt a fixed schedule with a total of 
$T = R E$ local updates, where $R$ is the number of global rounds and $E$ is the number of local steps per global round. In this setting, the server waits for all 
clients to complete $E$ local updates before performing aggregation. 
In contrast, we propose that rather than enforcing aggregation at predetermined time instants, it is triggered based on system demands or application events such as wireless channel conditions, client availability, or application-level demands, without adhering to a rigid synchronization schedule.
Specifically, in each round $r$, the clients perform their local updates sequentially, and
aggregation is performed after $\tau_r E$ local updates have been executed,
where $\tau_r \in (0,1]$ is a tunable parameter that determines the aggregation
point within round $r$. We define the set of aggregation time instances as $\mathcal{T}_{\mathrm{agg}} = \{ t_r \}_{r=1}^R, \text{with } t_r = t_{r-1} + \tau_r E,\; t_0 = 0.$
%This formulation allows aggregation to be triggered adaptively in response to wireless channel conditions, client availability, or application-level demands, without adhering to a rigid synchronization schedule.
\newline
The following lemma states that since each aggregation round in CHARGE-FL uses at most a fraction $\tau_r \leq 1$ of the maximum local steps $E$, the total number of iterations $T_c$ cannot exceed the fixed-schedule baseline $T = R E$.

\begin{lemma}\label{updateLemma}
   Let $T_c=\sum_{r=1}^R \tau_r E$ represent the number of iterations of updating the global model using CHARGE-FL, and the aggregation frequency parameter $\tau_r$ satisfies the condition $ \sum_{r=1}^R \tau_r\leq R$, then $T_c\leq T$.
\end{lemma}

\noindent Let $\mathcal{E}_t \subseteq [E]$ be the set of possible local steps at time $t$ and $\mathcal{S}_{t,e} \subseteq\{1,2,\dots, N\}$ the set of clients that completed exactly $e$ local steps before the aggregation deadline. We define $Q_{t} = |\mathcal{E}_{t}|$ and $N_{t,e} = |\mathcal{S}_{t,e}|$ . Given the above, we formulate the global optimization objective $F(\bmtheta)$ as:
\vspace{-1em}
\begin{equation} 
 \bmtheta^* = \arg\min_{\bmtheta} \left\{ F(\bmtheta) = \sum_{e \in \mathcal{E}_t}\frac{1}{N_{t,e}} \sum_{k \in \mathcal{S}_{t,e}}  f_k(\bmtheta) \right\}, 
\end{equation}
where $f_k(\bmtheta)$ is the local objective at the $k$-th client for all $k \in [N]$. Let $\bmtheta_{t,e}^k$ be the local model at client $k$ after $e$ local steps. The proposed server aggregation rule is given as: 
\begin{equation} 
\bmtheta_{t} \leftarrow \sum_{e \in \mathcal{E}_t} \frac{1}{Q_t N_{t,e} } \sum_{k \in \mathcal{S}_{t,e}} \bmtheta_{t,e}^k 
\label{eq:aggregationrule}
\end{equation} 
Note that the above aggregation and problem formulation inherently ensure participation of all clients, irrespective of the number of local steps.
\begin{remark} \label{expectationLemma} The aggregation scheme satisfies the expectation condition:
$\sum_{e\in\mathcal{E}_t} \frac{1}{Q_t N_{t,e} } \sum_{k \in \mathcal{S}_{t,e}}1  = 1.$
\end{remark}

\noindent The above result states that the aggregation rule in \eqref{eq:aggregationrule} is appropriately normalized. This ensures that the aggregated model is an unbiased convex combination of client updates and prevents over-representation of faster clients and under-representation of slower clients, thus helping to mitigate client drift due to heterogeneous training progress.

\vspace{-3mm}
\subsection{Dual Purpose Precoding}
\label{ssec:dualprecoding}
To minimize the impact of channel noise under a power constraint $P$, we introduce a dynamic precoding scheme that adapts to each client’s computation level so that the clients that complete fewer local steps require a different precoding to ensure fair contribution. In CHARGE-FL, we consider two variants for uplink transmission under the average power budget $P$, both designed to compensate for partial local progress.  

\noindent \textbf{Scheme I (local step-wise transmission)}: Let $\bmtheta_{t,0}^k$ denote the $k$-th client’s local model at time $t$, and define $\Delta\bmtheta_{t,e}^k \triangleq \bmtheta_{t,e}^k - \bmtheta_{t,0}^k$. Each client in $\mathcal{S}_{t,e}$ transmits $\vecx_{t,e}^k =\sqrt{\alpha_{t,e}}\,\Delta\bmtheta_{t,e}^k$, where the common scaling factor is given as
\begin{align}
\alpha_{t,e} = \frac{P}{\max\limits_{k\in\mathcal{S}_{t,e}} \mathbb{E}\!\left[\|\Delta\bmtheta_{t,e}^k\|^2\right]},
\label{eq:perEalpha}
\end{align}
with the expectation taken over local SGD randomness. This guarantees $\mathbb{E}[\|\vecx_{t,e}^k\|^2]\le P$ for all $k \in \mathcal{S}_{t,e}$. The server receives $\vecy_{t,e} = \sqrt{\alpha_{t,e}}\sum_{k\in\mathcal{S}_{t,e}}\Delta\bmtheta_{t,e}^k + \tilde{\vecw}_{t,e}$, where $\tilde{\vecw}_{t,e}\sim\mathcal{N}(0,\sigma_w^2\mathbf{I}_d),$ and reconstructs the noisy aggregate as  
\begin{align}
\tilde \bmtheta_{t,e} = \frac{\vecy_{t,e}}{N_{t,e}\sqrt{\alpha_{t,e}}} + \bmtheta_{t,0} 
= \frac{1}{N_{t,e}}\sum_{k\in\mathcal{S}_{t,e}}\bmtheta_{t,e}^k + \vecw_{t,e},
\label{eq:thetate}
\end{align} 
with $\vecw_{t,e}\sim\mathcal{N}(0,\tfrac{\sigma_w^2}{N_{t,e}^2\alpha_{t,e}}\mathbf{I}_d)$. The global model is updated as  
\begin{align}
\bmtheta_t = \sum_{e\in\mathcal{E}_t} \frac{1}{Q_t N_{t,e}} \sum_{k \in \mathcal{S}_{t,e}} \bmtheta_{t,e}^k + \vecw_t, 
\label{eq:server_global}
\end{align} 
where $\vecw_t \triangleq \sum_{e\in\mathcal{E}_t} \tfrac{\tilde\vecw_{t,e}}{Q_t N_{t,e} \sqrt{\alpha_{t,e}}}$.  
This scheme provides fine-grained aggregation but requires up to $E$ separate uplink communications per round.  

\noindent \textbf{Scheme II (single-shot transmission)}: To reduce communication cost in Scheme I, we approximate the per-local step scaling factors by their average,  
\vspace{-2pt}
\begin{align}
\alpha_t = \frac{1}{Q_t}\sum_{e \in \mathcal{E}_t} \alpha_{t,e}
\label{eq:avgAlpha}
\end{align} 
\vspace{-2pt}and apply this common factor across all client groups $\mathcal{S}_{t,e}$. Each client then transmits only once per round using $\alpha_t$, eliminating the need for multiple transmissions while still preserving the power constraint. Unlike COTAF, the computation for this precoding takes into account the different local step-wise precoding factors by averaging them, thereby resulting in a unified precoding factor. The server receives the update as:
\begin{align}
\vecy_{t,e} = \sqrt{\alpha_{t}}\sum_{k\in\mathcal{S}_{t,e}}\Delta\bmtheta_{t,e}^k + \tilde{\vecw}_{t,e},
\label{eq:ytenew}
\end{align}
which is decoded similarly to the technique stated in the previous scheme using $\alpha_t$.
 
 \noindent \textbf{Fading Channels:} We consider a block-fading channel for client $k$ at time $t$, modeled as $\bar h_t^k = h_t^k e^{j\Omega_t^k}$, where $h_t^k$ and $\Omega_t^k$ denote the fading magnitude and phase, respectively. When $h_t^k$ is small,  noise dominates, degrading performance. To address this, we employ the threshold-based selection where clients with $h_t^k > \hat h$ transmit \cite{sery2021over}, using the precoded signal given as $\vecx_{t,e}^k = \frac{\hat h \sqrt{\alpha_t}}{h_t^k} e^{-j\Omega_t^k}\Delta \bmtheta_{t,e}^k$, 
while clients below the threshold do not participate, leading to the partial participation-based FL induced due to $\hat h$.

\begin{algorithm}[t]
\SetAlgoLined
\KwIn{Clients $N$, max local steps $E$, aggregation rounds $\mathcal{T}_{\text{agg}}$, Initialization $\bmtheta_0$.}
\For{$t=0,\ldots,T$}{
    Each client $k \in [N]$ performs local SGD on $\bmtheta_t^k$.\\
    \If{$t \in \mathcal{T}_{\text{agg}}$}{
        \For{$e \in \mathcal{E}_t$}{
            Each $k \in \mathcal{S}_{t,e}$ computes update 
            $\Delta \bmtheta_{t,e}^k = \bmtheta_{t,e}^k - \bmtheta_{t,0}^k$.\\
            Compute precoding factor $\alpha_{t,e}$ via \eqref{eq:perEalpha} (Scheme I) or $\alpha_t$ via \eqref{eq:avgAlpha} (Scheme II).\\
            Transmit precoded signal $\vecx_{t,e}^k$ over MAC.\\
        }
        Server decodes $\bmtheta_{t,e}$ using \eqref{eq:thetate}, 
        aggregates via \eqref{eq:server_global}, and updates $\bmtheta_t$.\\
        Server broadcasts $\bmtheta_t$; each client synchronizes $\bmtheta_t^k \gets \bmtheta_t$.\\
    }
}
\KwOut{Final global model $\bmtheta_T$}
\caption{CHARGE-FL (with precoding Scheme I or II)}
\label{tab:AlgTable}
\end{algorithm}

\section{Convergence Analysis and Discussions}
\label{sec:convganalysis}

We now present a convergence analysis for the CHARGE-FL proposed in the previous section. We define the degree of data heterogeneity as
$\Gamma \triangleq F^* - \tfrac{1}{N}\sum_{k=1}^N f_k^*$, where $f_k^* \triangleq \min_{\bmtheta} f_k(\bmtheta),$ and let $\delta_t\triangleq\mathbb{E}\big[\|\bmtheta_t-\bmtheta^*\|^2\big]$. We make the following assumptions:

\noindent [\textbf{AS1}] Each $f_k$ is $L$-smooth and $\mu$-strongly convex.

\noindent [\textbf{AS2}] For any iterate $\bmtheta$ and sample $j_t^k$ used by client $k$,  
\[
 \mathbb{E}\big[\|\nabla f_{j_t^k}(\bmtheta)\|^2\big]\le G^2, \quad
 \mathbb{E}\big[\|\nabla f_{j_t^k}(\bmtheta)-\nabla f_k(\bmtheta)\|^2\big]\le M_k^2.
\]

\begin{lemma}[Bound on precoding factor]
Suppose AS2 holds with step-sizes $\eta_s>0$ and $\eta_{\max} \triangleq \max_s\eta_s$. For $\Delta\bmtheta_{t,e}^k \triangleq \bmtheta_{t,e}^k - \bmtheta_{t,0}^k$, the precoding factor satisfies
$\mathbb{E}\big[\|\Delta\bmtheta_{t,e}^k\|^2\big] \le e^2\eta_{\max}^2 G^2,$
and hence $\tfrac{1}{\alpha_{t,e}} \le \tfrac{e^2\eta_{\max}^2 G^2}{P}$.
\end{lemma}

\noindent Thus, $\alpha_{t,e}$ both enforces the power constraint and amplifies smaller updates, ensuring that clients with fewer local steps remain fairly represented in the aggregate. The proofs of the lemmas, theorems, and some of the constants are delegated to the supplementary.

\begin{theorem}
    \label{thm:convergence}
Let $\delta_t\triangleq\mathbb{E}[\|\bmtheta_t-\bmtheta^*\|^2]$.  
Suppose AS1--AS2 hold, and set $\eta_t=\tfrac{2}{\mu(\gamma+t)}$ with $\gamma \geq \max(\tfrac{8L\rho}{\mu},E)$ and $\rho>\tfrac{1}{\mu}$. Then:
\vspace{-2mm}
\begin{small}
 \begin{align}
\mathbb{E}[F(\bmtheta_t)] - F(\bmtheta^*) 
&\leq 
\sum_{e\in\mathcal{E}_t}\frac{e^2}{\Xi_{t,e}^2 Q_t}\,
\frac{2L \max\!\left(4\tilde{C}, \mu^2 \gamma \delta_0\right)}{\mu^2 (t + \gamma)},
\end{align}
\end{small}
where $\tilde{C} = \hat{B} + \tfrac{4d\sigma^2_w G^2}{P\hat{h}}$, and $\hat{B}$ depends on gradient variance, client heterogeneity, and local step distribution. $\Xi_{t,e}=N_{t,e}$ in the no-fading case and $\Xi_{t,e}=K_{t,e}$ under fading with threshold $\hat{h}$. In fading channels, an additional degradation arises as
$\tilde{C} \;\gets\; \tilde{C} + \tilde{D}$, where $\tilde{D} = \tfrac{4(N-K)G^2}{N-1}\sum_{e\in\mathcal{E}_t} K_{t,e}$ with $K$ denoting the active clients at time $t$.
\end{theorem}

\noindent Theorem~\ref{thm:convergence} shows that CHARGE-FL achieves an $\mathcal{O}(1/T)$ convergence rate in both ideal and fading channels. In the baseline case, the error depends on $\sigma_w^2/P$ and the heterogeneity factor $\sum_{e\in\mathcal{E}_t} e^2/(N_{t,e}^2Q_t)$, which vanishes when all clients complete the same number of local steps, reducing CHARGE-FL to COTAF. In fading, the additional terms $\tilde{D}$ and $\tfrac{4d\sigma^2_w G^2}{P\tilde{h}}$ capture the effects of limited client participation and higher effective noise. Despite these impairments, CHARGE-FL preserves the same $\mathcal{O}(1/T)$ rate while balancing heterogeneous client updates and channel variability.
  
\noindent From \eqref{eq:server_global}, the cumulative noise power is given as
\vspace{-2mm}
\[
\mathbb{E}\big[\|\vecw_t\|^2\big]
= \frac{d\,\sigma_w^2}{Q_t^2}\sum_{e\in\mathcal{E}_t}\frac{1}{N_{t,e}^2\alpha_{t,e}}
\;\le\; \frac{d\,\sigma_w^2}{Q_t}\max_{e\in\mathcal{E}_t}\frac{1}{N_{t,e}^2\alpha_{t,e}}.
\vspace{-2mm}
\]
Hence, for fixed per-group scaling, larger $Q_t$ yields stronger attenuation of cumulative noise. When all clients perform $E$ local steps (so $Q_t=1$), the noise matches that of standard COTAF. With heterogeneous local steps, $Q_t>1$ provides additional suppression, allowing CHARGE-FL to (a) reduce cumulative noise under partial participation, and (b) automatically adapt the effective SNR to client participation without altering transmission power.
\vspace{-4mm}
\section{Simulation Results and Discussion}
\label{sec:simresult}

In this section, we present experimental results of CHARGE-FL against relevant baselines. 

\noindent \textbf{Datasets and Model:} We conduct experiments on the CIFAR-10 and CIFAR-100~\cite{krizhevsky2009learning} datasets using a custom-built CNN with three convolutional and two fully connected layers. The training datapoints are split among clients under both IID and non-IID partitions. For the non-IID partition, we adopt a label-skew strategy~\cite{sery2021over} with similarity parameter $\beta=0.5$ unless specified. Unless otherwise noted, we use $30$ clients, mini-batch size $64$, $150$ communication rounds, and a MAC channel with $0$ dB SNR. 
%The precoding factor $\alpha_{t,e}$ is implemented by upper bounding $\max_{k\in \mathcal{S}_{t,e}}\mathbb{E}[\|\bmtheta_{t,e}^k-\bmtheta_{t,0}^k\|^2]$ with $E^2\eta_t G^2$ as in \cite{sery2021over}, where $G$ is estimated empirically from gradient norms.  

\noindent \textbf{Baselines:} We compare with COTAF~\cite{sery2021over}, which uses fixed aggregation and excludes stragglers, and with NoisyProx (a FedProx variant~\cite{li2020federated}), which lacks precoding to mitigate channel noise.

\begin{figure}[htbp]
    \centering
    \includegraphics[scale=0.185]{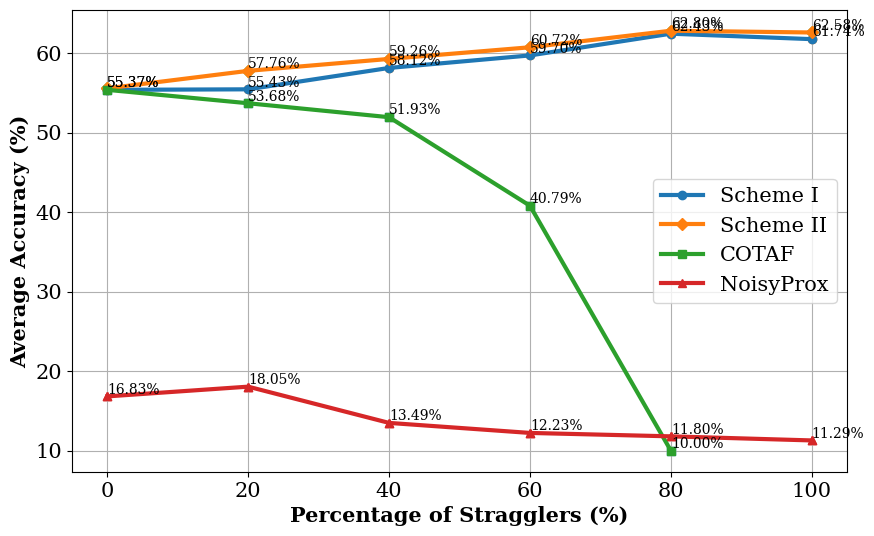}
    \includegraphics[scale=0.185]{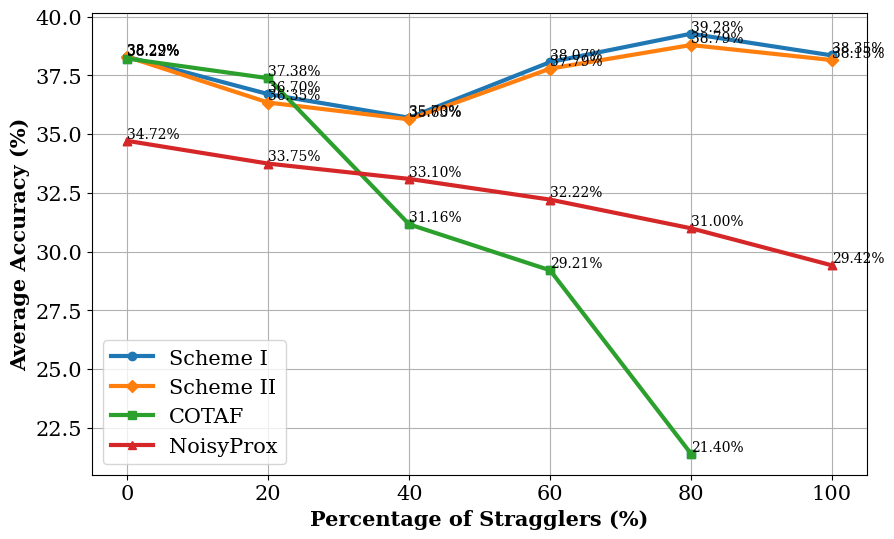}
    \includegraphics[scale=0.185]{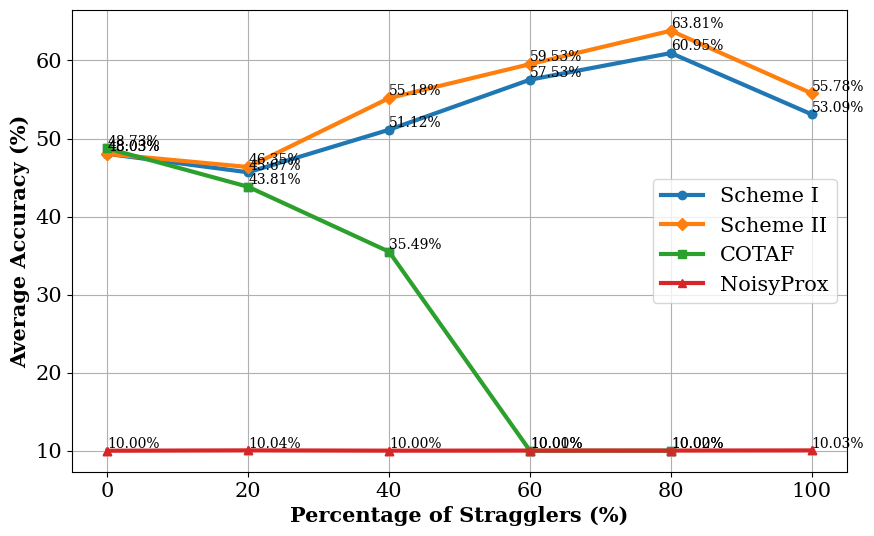}
     \includegraphics[scale=0.185]{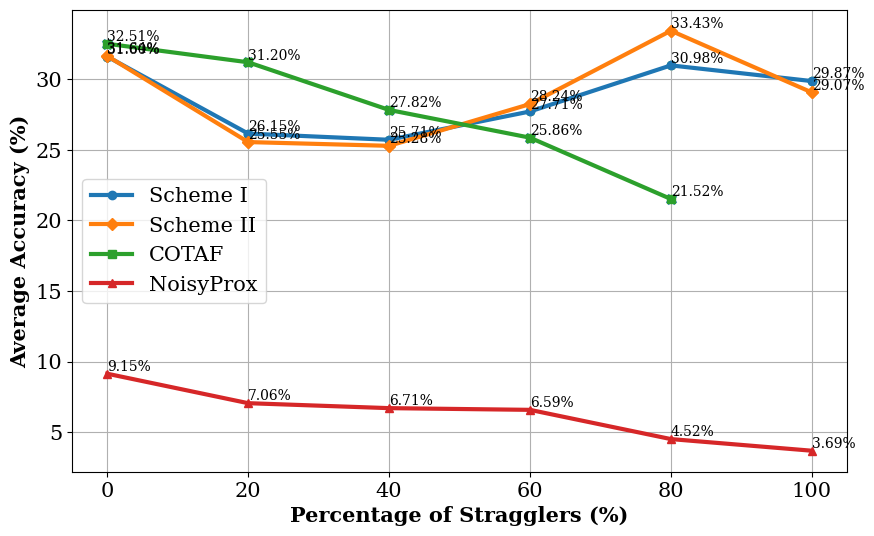}
    \caption{The average accuracy of CHARGE-FL with baseline algorithms for AWGN (top) and fading channel (bottom) on the CIFAR-10 (left) and CIFAR-100 (right) datasets.}
    \vspace{-3mm}
    \label{fig:sim05c10c100}
\end{figure}

\noindent To assess CHARGE-FL in adaptive settings, Fig.~\ref{fig:sim05c10c100} compares it with baselines under varying straggler levels. The top row shows average accuracy on CIFAR-10 (left) and CIFAR-100 (right) with noisy channels, and the bottom row with fading channels, all under non-IID data ($\beta=0.5$). At $20\%$ stragglers, CHARGE-FL shows a slight accuracy drop due to uneven local step completion, where small groups yield low normalization factors $Q_tN_{t,e}$. As the straggler fraction increases, CHARGE-FL consistently outperforms baselines, highlighting its robustness to heterogeneous client progress. 
\vspace{-3mm}
\subsection{Varying data heterogeneity}
\label{ssec:heterogeneity}

\begin{figure}[h]
    \centering
    \includegraphics[scale=0.185]{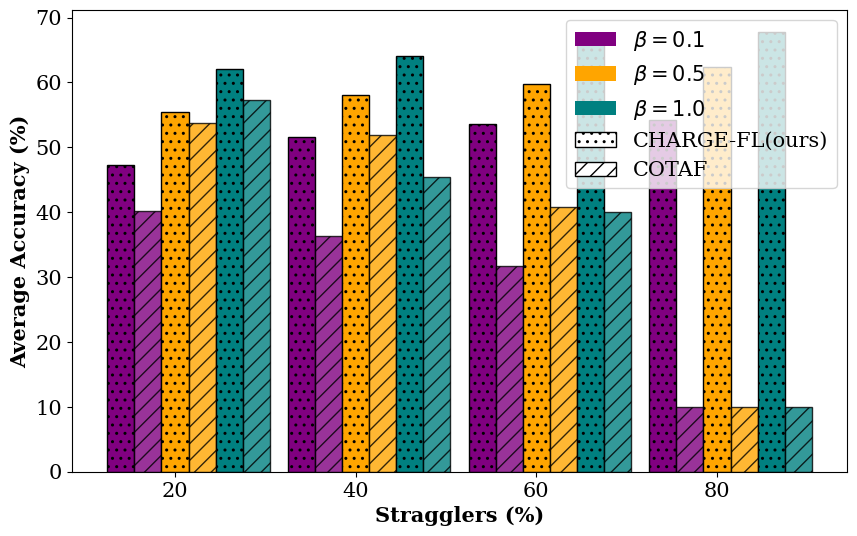}
    \includegraphics[scale=0.185]{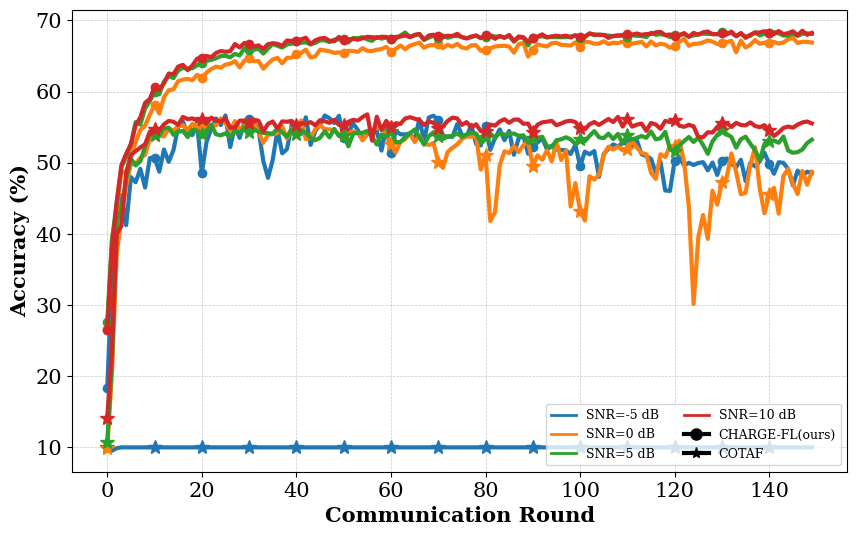}
    \caption{The average accuracy of CHARGE-FL (Scheme II) and COTAF on the CIFAR-10 dataset across varying data heterogeneity (left) and varying SNR (right).}
    \label{fig:varyinghet}
\end{figure}
\noindent We examine the impact of data heterogeneity on CHARGE-FL (Scheme II) as compared to COTAF. In Fig.~\ref{fig:varyinghet} (left), we observe the accuracy on CIFAR-10 across heterogeneity levels across varying similarity parameter $\beta$ values and straggler percentages. Lower $\beta$ indicates higher heterogeneity. The results show that CHARGE-FL consistently outperforms COTAF, with particularly large gains under high heterogeneity ($\beta=0.1$), demonstrating its robustness.
\vspace{-3mm}
\subsection{Varying SNR}
\label{ssec:snr}

We evaluate the effect of SNR on performance by comparing CHARGE-FL (Scheme II) with COTAF on CIFAR-10 (Fig.~\ref{fig:varyinghet}, right). CHARGE-FL consistently achieves higher accuracy across SNR levels. While accuracy drops at very low SNR ($-5$ dB), it still converges reliably, unlike COTAF which fails under such conditions. 
\vspace{-1em}
\subsection{Varying Number of local steps at Stragglers}
\label{ssec:epochdist}
\begin{figure}
    \centering
    \includegraphics[scale=0.185]{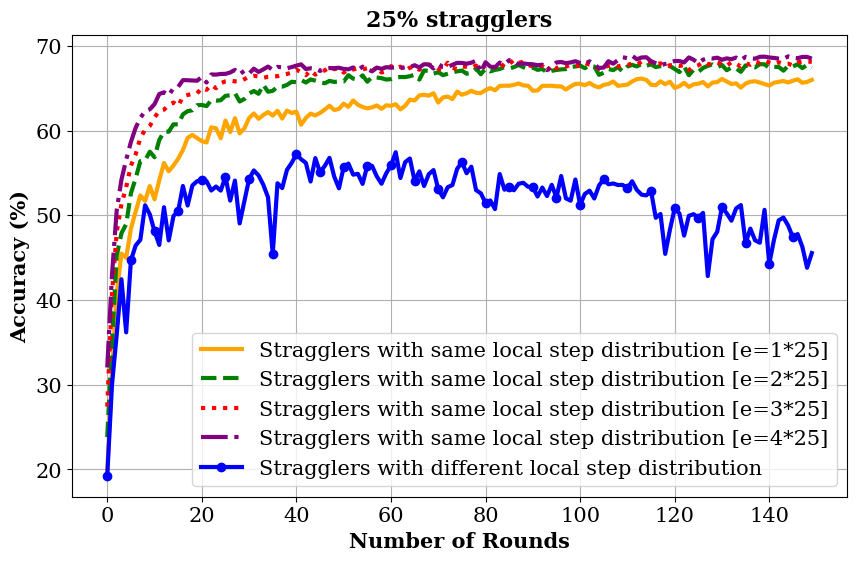}
    \includegraphics[scale=0.185]{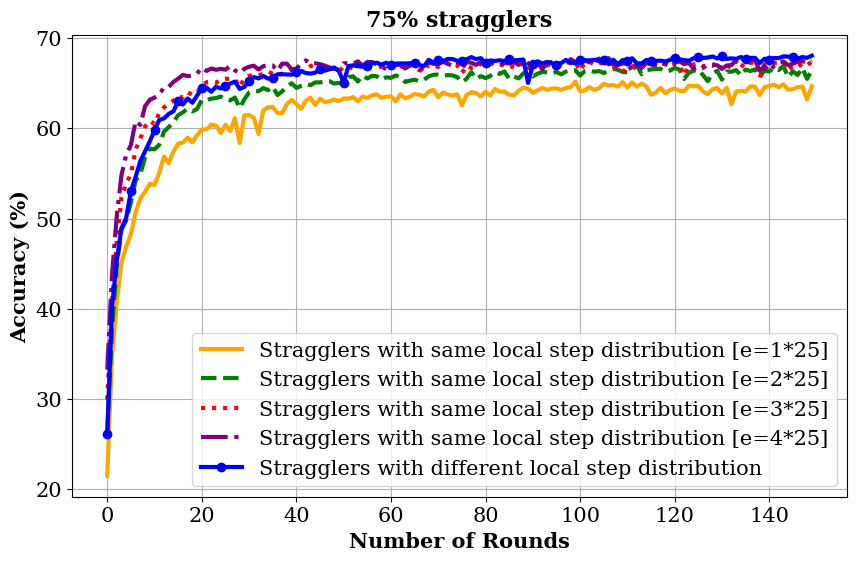}
    \caption{Accuracy performance of CHARGE-FL (Scheme II) on CIFAR-10 for varying local step distribution of stragglers.}
    %\vspace{-5mm}
    \label{fig:straggler_effect}
\end{figure}

We analyze the impact of stragglers on CHARGE-FL (Scheme II). As shown in Fig.~\ref{fig:straggler_effect} (left), when stragglers perform the same number of local steps, the method attains higher accuracy with stable convergence. However, with heterogeneous local steps, noise amplification leads to poor performance. Fig.~\ref{fig:straggler_effect} (right)  shows that as the number of stragglers grows, the accuracy loss diminishes because larger local step groups increase the normalization factor $Q_tN_{t,e}$ and suppress noise, highlighting the robustness of CHARGE-FL.
\vspace{-1.5em}
\section{Conclusions}
\label{sec:conclusions}
\vspace{-1em}
In heterogeneous OTA-FL systems, we propose the CHARGE-FL framework that allows for a flexible aggregation mechanism and a heterogeneity-aware precoding scheme that enables robust and efficient FL in wireless environments. By decoupling aggregation from fixed schedules and dynamically adapting transmission strategies based on channel availability and client readiness, CHARGE-FL effectively addresses both communication noise and variability in the computational capability of the clients. We observed from our experiments that CHARGE-FL reduces average execution time by approximately $30\%$ as compared to COTAF while providing equally efficient convergence. Experimentally, we demonstrate that the proposed algorithm outperforms the baselines under varying levels of SNR and percentage of stragglers.  
% \newpage

\bibliographystyle{IEEEbib}
\bibliography{strings,refs}

\newpage
\section{Appendix}
We start by defining
\begin{equation}
    \mathbf{g}_t\triangleq \sum_{e\in\mathcal{E}_t}\frac{1}{N_{t,e} Q_t}\sum_{k\in \mathcal{S}_{t,e}}(\nabla f_{j_t^n}(\bmtheta_{t,e}^k)
    \end{equation} 
    \begin{equation}
    {\mathbf{\bar g}_t}\triangleq \sum_{e\in\mathcal{E}_t}\frac{1}{N_{t,e} Q_t}\sum_{k\in \mathcal{S}_{t,e}}\nabla f(\bmtheta_{t,e}^k). \tag{1.4}
    \label{eq:1.4}
\end{equation} 
As SGD iterations are carried out on the indices $j_t^n$ which are distributed uniformly, it holds that $\Ex [\mathbf{g}_t]=\mathbf{\bar{g}_t}$.
We write 
\begin{align}
\bar\vecw_t \triangleq \vecw_{t+1}\mathbbm{1}_{t+1 \in H} - \vecw_t\mathbbm{1}_{t \in H}
\end{align}

The above noise vector satisfies

\begin{align}
    \Ex[\norm*{\bar{\vecw}_t}^2] &=  \Ex[\norm*{\vecw_{t+1}\mathbbm{1}_{t+1 \in H}}^2] +  \Ex[\norm*{\vecw_t\mathbbm{1}_{t \in H}}^2] \nonumber\\
    &= \sum_{e\in\mathcal{E}_t}\frac{d\sigma^2}{Q_t^2N_{t,e}^2\alpha_{{t+1},e}}\mathbbm{1}_{t+1 \in H}+\sum_{e\in\mathcal{E}_t}\frac{d\sigma^2}{Q_t^2N_{t,e}^2\alpha_{{t},e}}\mathbbm{1}_{t \in H}\nonumber\\
    &\leq d\sigma^2_w\sum_{e\in\mathcal{E}_t} \frac{1}{N_{t,e}^2 Q_t \min(\alpha_{{t+1},e},\alpha_{t,e})}\mathbbm{I}_t
\end{align}

Here, $Q_t$ is the sum of the binary set $\hat{\mathcal{Q}}_t \gets \{\hat{Q}_i \mid \hat{Q}_i$ is 1 if at least one client sent the update after $i^{th}$ epoch.
\begin{lemma}
    Let $\{\bmtheta^k_t\}$ be the weights from each client k and $\{\overline{\bmtheta_t}\}$ be the weights obtained when the weights trained by the clients are aggregated over the true channel. Then, when Assumption 1 and 2 are met and the SGD step size satisfies $\eta_t \leq \frac{1}{4L}$, 

    \begin{align}
&\Ex \big[\|\bar{\bmtheta}_{t+1} - \bmtheta^*\|^2 \big] 
\leq (1 - \mu \eta_t) \Ex\big[\|\bar{\bmtheta}_t - \bmtheta^*\|^2 \big]+ \nonumber \\
& \eta_t^2 \Ex \big[\|g_t - \bar{g}_t + \frac{\bar{\vecw}_t}{\eta_t} \|^2 \big]
- \frac{3}{2} \eta_t \Ex \big[F(\bar{\bmtheta}_t) - F^*\big] + 
\nonumber \\
& \sum_{e\in\mathcal{E}_t}\frac{2}{N_{t,e} Q_t} \sum_{k\in\mathcal{S}_{t,e}} \Ex \big[\|\bar{\bmtheta}_t - \bmtheta_{t,e}^k\|^2 \big] 
+ 6 L \eta_t^2 \Gamma.
\end{align}

\end{lemma}
\begin{lemma}
    When the step size sequence $\{\eta_t\}$ consists of decreasing positive numbers satisfying $\eta_t \leq 2\eta_{t+E}$ for all $t\geq0$ and Assumption 3 is satisfied, then

    \begin{align}
\Ex \bigg[\norm*{ g_t - \bar{g}_t + \frac{\bar{\vecw}_t}{\eta_t}}^2 \bigg]
&\leq \sum_{e\in\mathcal{E}_t}\frac{1}{N_{t,e}^2 Q_t^2}\sum_{k\in \mathcal{S}_{t,e}} M_k^2 \nonumber\\
&+ \frac{4d\sigma^2_wG^2\mathbbm{I}_t}{P Q_t} \sum_{e\in\mathcal{E}_t} \frac{e^2}{N_{t,e}^2  }.
\end{align}

\end{lemma}
\textbf{Lemma A2}
Following till equation A.12 of cotaf, we get

\begin{align}
 \Ex[\norm*{\frac{\bar{\vecw}_t}{\eta_t}}^2] \leq \frac{d\sigma^2_w\mathbf{I}_t}{\eta_t^2}\sum_{e\in\mathcal{E}_t} \frac{1}{N_{t,e}^2 Q_t\min(\alpha_{{t+1},e},\alpha_{t,e})}
\end{align}

Bounding 
$ \alpha_{t,e}^k \triangleq \frac{P}{\max_{k\in \mathcal{S}_{t,e}} \Ex \left[\norm*{  \bmtheta_{t,e}^k - \bmtheta_{t-e}^k }^2 \right]}$ via:
\begin{align}
    \frac{1}{\alpha_{t,e}^k}&\underset{\text{(i)}}{\leq} \frac{1}{P}\max \Ex\left[\norm*{\sum_{t'=t-e}^{t-1}\eta_{t'}\nabla f(\bmtheta_{t'}^k)}^2\right] \nonumber\\
    &  \underset{\text{(ii)}}{\leq} \frac{1}{P} \max \left[ e \eta_{t-e}^2\sum_{t'=t-e}^{t-1} \Ex\norm*{\nabla f(\bmtheta_{t'}^k)}^2\right]\nonumber\\
    &\underset{\text{(iii)}}{\leq} \frac{1}{P}e^2\eta_{t-e}^2 G^2 \underset{\text{(iv)}}{\leq} \frac{1}{P}4e^2\eta_t^2 G^2
\end{align}
where \text{(i)} follows from \eqref{eq:localSGD},  \text{(ii)} follows from inequality $\norm*{\sum_{t'=t-e}^e r_t}^2 \leq e\sum_{t'=t-e}^e \norm*{r_t}^2$, \text{(iii)} holds by AS2 and  \text{(iv)} holds as $\eta_t \leq 2 \eta_{t+e}$ for all $t>0$.
Now 
\begin{align}
    \frac{1}{\min(\alpha_{t+1}^e,\alpha_t^e)}&=\max\left( \frac{1}{\alpha_{t,e}},\frac{1}{\alpha_{{t+1},e}}\right)\nonumber\\
    &\leq \frac{1}{P}4e^2G^2\eta_{t}^2
\end{align}

Therefore
\begin{align}
 \Ex[\norm*{\frac{\bar{\vecw}_t}{\eta_t}}^2] \leq \frac{4d\sigma^2_wG^2\mathbbm{I}_t}{P Q_t} \sum_{e\in\mathcal{E}_t} \frac{e^2}{N_{t,e}^2  }
\end{align}

\begin{lemma}
    When the step size sequence $\{\eta_t\}$ consists of decreasing positive numbers satisfying $\eta_t \leq 2\eta_{t+E}$ for all $t\geq0$ and Assumption 3 holds then,
     \begin{align}
\Ex \bigg[ \norm*{\bar{\bmtheta}_t - \bmtheta_{t,e}^k }^2 \bigg]
&\leq 4\eta^2_tG^2 \sum_{e\in\mathcal{E}_t} e^2
\end{align}
\end{lemma}
Now to prove the theorem, we first define $\delta_t \triangleq \Ex[\lVert\bar\bmtheta_t-\bmtheta^*\rVert^2]$. Then from Lemmas we have

 \begin{align*}
    \delta_{t+1} &\leq (1-\mu\eta_t)\delta_t\nonumber\\
    &+ \eta_t^2  \left(\sum_{e\in\mathcal{E}_t}\frac{1}{N_{t,e}^2 Q_t^2}\sum_{k\in \mathcal{S}_{t,e}} M_k^2+ \frac{4d\sigma^2_wG^2\mathbbm{I}_t}{P Q_t} \sum_{e\in\mathcal{E}_t} \frac{e^2}{N_{t,e}^2}\right) \nonumber\\
    & - \frac{3}{2} \eta_t \Ex[F(\bar\bmtheta_t)-F^*] + \frac{8\eta_t^2 G^2}{Q_t} \sum_{e\in\mathcal{E}_t}e^2 + 6L\eta_t^2 \Gamma.\tag{1.16}
    \label{eq:1.16}
\end{align*}
The upper bound in \eqref{eq:1.16} can be reformulated as \cite{sery2021over}
\begin{equation}
     \delta_{t+1} \leq (1-\mu\eta_t)\delta_t + \eta_t^2C, \tag{1.17}
     \label{eq:1.17}
\end{equation}
where 
\begin{align}
    C=B + \frac{4d\sigma^2_wG^2\mathbbm{I}_t}{P Q_t} \sum_{e\in\mathcal{E}_t} \frac{e^2}{N_{t,e}^2}
    \end{align}
    and 
\begin{align}
    B\triangleq \frac{8 G^2}{Q_t} \sum_{e\in\mathcal{E}_t}e^2+\sum_{e\in\mathcal{E}_t}\frac{1}{N_{t,e}^2 Q_t^2}\sum_{k\in \mathcal{S}_{t,e}} M_k^2+6L\Gamma.
 \end{align}
 
 \begin{align}
\mathbb{E}[F(\bmtheta_t)] - F(\bmtheta^*) \leq \sum_{e\in\mathcal{E}_t}\frac{e^2}{N_{t,e}^2 Q_t}
\frac{2L \max\left(4\hat{C}, \mu^2 \gamma \delta_0\right)}{\mu^2 (t + \gamma)}.
\end{align}
where 
\begin{align}
    \hat{C}=\hat{B} + \frac{4d\sigma^2_w G^2}{P}
    \end{align}
    and 
\begin{align}
    \hat{B} &= 8 G^2 \sum_{e\in\mathcal{E}_t}N_{t,e}^2+\sum_{e\in\mathcal{E}_t}\frac{1}{e^2 Q_t}\sum_{k\in \mathcal{S}_{t,e}} M_k^2 \notag\\ &+ 6L\Gamma \left(\sum_{e\in\mathcal{E}_t}\frac{e^2}{N_{t,e}^2 Q_t}\right)^{-1} 
 \end{align}

\begin{lemma}
Assume that the server collects the outputs of the first \(K\) responded devices such that the aggregation step of the server is $\bmtheta_{t+E} \leftarrow \frac{N}{K} \sum_{k\in \mathcal{S}_t} p_k \bmtheta_{t+E}^k $ where $\mathcal{S}_t$ is the set of K devices in the $t^{th}$ iteration, then 
\begin{align} \Ex[\frac{N}{K} \sum_{k\in \mathcal{S}_t} p_k] = 1\end{align}
\end{lemma}

This lemma is a known result \cite{li2019convergence} and establishes the unbiasedness of the aggregation step. The proof of this lemma is given below for reference and for further guidance. 

\begin{proof}
Let K be the number of clients randomly selected out of the total N clients. We have to show that $\Ex[\frac{N}{K} \sum_{k\in \mathcal{S}_t} p_k] = 1$ The randomness in the equation is present in the selection of the clients in each round. We can capture this randomness using a variable as follows.

\begin{align} \frac{N}{K} \sum_{k\in \mathcal{S}_t} p_k =  \frac{N}{K} \sum_{k=1}^N p_k \mathbbm{1}(k \in \mathcal{S}_t)\end{align}
where \begin{align}
    \mathbbm{1}(k \in \mathcal{S}_t)= 
\begin{cases}
    1,& \text{if } k \in \mathcal{S}_t\\
    0,              & \text{else}
\end{cases}
\end{align}
Here the randomness is captured by the indicator function and is the only source of randomness.

To find the expectation, we know that $\Ex[\mathbbm{1}(k \in \mathcal{S}_t)]=P(k \in \mathcal{S}_t)$.

If we sample the clients uniformly random, then $P(k \in \mathcal{S}_t) = \frac{K}{N}$. 

Using these facts we can prove that the expectation is 1.
\begin{align}\Ex[\frac{N}{K} \sum_{k\in \mathcal{S}_t} p_k]=\Ex[\frac{N}{K} \sum_{k=1}^N p_k \mathbbm{1}(k \in \mathcal{S}_t)]\end{align}
\begin{align} = \frac{N}{K} \sum_{k=1}^N p_k \Ex[\mathbbm{1}(k \in \mathcal{S}_t)]\end{align}
\begin{align} = \frac{N}{K} \sum_{k=1}^N p_k P((k \in \mathcal{S}_t))\end{align}
\begin{align} = \frac{N}{K} \sum_{k=1}^N p_k \frac{K}{N}\end{align}
\begin{align} =\sum_{k=1}^N p_k\ = 1 \end{align}
Hence, $\Ex[\frac{N}{K} \sum_{k\in \mathcal{S}_t} p_k] = 1$.

\end{proof}

\subsection{Proposed Aggregation Equation}
We propose a system that does not wait for each client to finish running E-local epochs. Some clients might have very low computing power or might be doing other tasks which can slow down the model updates. If we wait for each client to finish model updates then it slows down the process quite a lot. 

Instead of waiting for each client to finish their local epochs, we only wait for the time required for each client to finish at least one local epoch, and then the server aggregates the updates from the clients. The clients send the update from the last epoch finished by them. The aggregation step for this system is as follows.

\textbf{Aggregation Equation Formulation}

\begin{proof}
Let the simple aggregation step be given by $\bmtheta_t \leftarrow \sum_{k=1}^N p_k \bmtheta_{t,e}^k $. To capture the randomness in this equation, we model $\bmtheta_{t,e}^k = \bmtheta_t^k \vecx_t^k$, where $\vecx_t^k$ is a one hot encoded vector which is 1 only at the index for the last complete local epoch and $\bmtheta_t^k$ is the local update. The size of $\bmtheta_t^k$ is $E\times d$, representing the model parameters for each epoch, and each column represents the model parameters for epoch e. The size of $\vecx_t^k$ is $E\times 1$. 

If we consider $\vecx_t$ as a matrix of size ExN where each column is $x_k^t$, then we have a matrix of the form $[x_1^t x_2^t   ...  x_N^t] = \vecx_t$

Taking the expectation of this equation, we get, 
\begin{align} \Ex[\sum_{k=1}^N p_k \bmtheta_t^k \vecx_t^k] = \sum_{k=1}^N p_k \bmtheta_t^k \Ex[\vecx_t^k]\end{align}
\begin{align}= \sum_{k=1}^N p_k \bmtheta_t^k \mu_t\end{align}
where $\mu_t$ is the mean of the $\vecx_t$ matrix, where $x_k^t$ is a multinomial variable. 
\begin{align}\mu_t = \begin{bmatrix}
           \frac{N_1}{N} \\
           \frac{N_2}{N} \\
           \vdots \\
           \frac{N_{t,e}}{N}
         \end{bmatrix}\end{align}

Hence,
\begin{align}
 \Ex[\sum_{k=1}^N p_k \bmtheta_t^k \vecx_t^k] &= \frac{N_1}{N}\sum_{k\in \mathcal{S}_1} p_k \bmtheta_{t,1}^{k}\ + \frac{N_2}{N}\sum_{k\in \mathcal{S}_2} p_k \bmtheta_{t,2}^{k}+ \\
 &.... + \frac{N_{t,e}}{N}\sum_{k\in \mathcal{S}_{t,E}} p_k \bmtheta_{t,E}^{k}   
\end{align}

\begin{align} \Ex[\sum_{k=1}^N p_k \bmtheta_t^k \vecx_t^k] = \sum_{e=1}^E\frac{N_{t,e}}{N}\sum_{k\in \mathcal{S}_{t,e}} p_k \bmtheta_{t,e}^{k}\end{align}

Further we can prove that \begin{align} \Ex[\sum_{Ne\neq 0, e=1}^{E}\frac{N}{|N_{t,e}|Q_t} \sum_{k\in \mathcal{S}_{t,e}} p_k] = 1\end{align}

\begin{align}
 &\Ex[\sum_{Ne\neq 0, e=1}^{E}\frac{N}{|N_{t,e}|Q_t} \sum_{k\in \mathcal{S}_{t,e}} p_k] = \Ex[\frac{N}{N_1 Q_t}\sum_{k\in \mathcal{S}_1} p_k\\  &+ \frac{N}{N_2 Q_t}\sum_{k\in \mathcal{S}_2} p_k+.... + \frac{N}{N_{t,e} Q_t}\sum_{k\in \mathcal{S}_{t,E}} p_k ] \nonumber\\
&= \Ex[\frac{N}{N_1 Q_t}\sum_{k=1}^N p_k \mathbbm{1}(k \in \mathcal{S}_1) + \frac{N}{N_2 Q_t}\sum_{k=1}^N p_k  \mathbbm{1}(k \in \mathcal{S}_2)+ \nonumber\\
&.... + \frac{N}{N_{t,e} Q_t}\sum_{k=1}^N p_k  \mathbbm{1}(k \in \mathcal{S}_{t,E})],  N_{t,e}\ne0
\end{align}

where \begin{align}
    \mathbbm{1}(k \in \mathcal{S}_{t,e})= 
\begin{cases}
    1,& \text{if } k \in \mathcal{S}_{t,e}\\
    0,              & \text{else}
\end{cases}
\end{align}

Here the randomness is captured by the indicator function and is the only source of randomness.

To find the expectation, we know that $\Ex[\mathbbm{1}(k \in \mathcal{S}_{t,e})]=P(k \in \mathcal{S}_{t,e})$.

If we sample the clients uniformly random then $P(k \in \mathcal{S}_{t,e}) = \frac{N_{t,e}}{N}$. 

\begin{align}   
&= \frac{N}{N_1 Q_t}\sum_{k=1}^N p_k \Ex[\mathbbm{1}(k \in \mathcal{S}_1)] + \frac{N}{N_2 Q_t}\sum_{k=1}^N p_k \Ex[\mathbbm{1}(k \in \mathcal{S}_2)]+\nonumber\\
&.... + \frac{N}{N_{t,e} Q_t}\sum_{k=1}^N p_k \Ex[\mathbbm{1}(k \in \mathcal{S}_{t,E})],  N_{t,e}\ne0\nonumber\\
&= \frac{N}{N_1 Q_t}\sum_{k=1}^N p_k P(\mathbbm{1}(k \in \mathcal{S}_1)) + \frac{N}{N_2 Q_t}\sum_{k=1}^N p_k  P(\mathbbm{1}(k \in \mathcal{S}_2))+ \nonumber\\
&.... + \frac{N}{N_{t,e} Q_t}\sum_{k=1}^N p_k P(\mathbbm{1}(k \in \mathcal{S}_{t,E})),  N_{t,e}\ne0\nonumber\\
&= \frac{N}{N_1 Q_t}\sum_{k=1}^N p_k  \frac{N_1}{N} + \frac{N}{N_2 Q_t}\sum_{k=1}^N p_k  \frac{N_2}{N}+\nonumber\\
&.... + \frac{N}{N_{t,e} Q_t}\sum_{k=1}^N p_k  \frac{N_{t,e}}{N},  N_{t,e}\ne0\nonumber\\
&= \frac{1}{Q_t}\sum_{k=1}^N p_k + \frac{1}{Q_t}\sum_{k=1}^N p_k+ .... + \frac{1}{Q_t}\sum_{k=1}^N p_k \nonumber\\
&= \frac{1}{Q_t} + \frac{1}{Q_t}+ .... + \frac{1}{Q_t}=1
\end{align}
 
Hence,  $ \Ex[\sum_{Ne\neq 0, e=1}^{E}\frac{N}{|N_{t,e}|Q_t} \sum_{k\in \mathcal{S}_{t,e}} p_k] = 1$

Therefore, looking at the expectation, we arrive at the aggregation rule, 
\begin{align} \bmtheta_{t} \leftarrow \sum_{Ne\neq 0, e=1}^{E}\frac{N}{|N_{t,e}|Q_t} \sum_{k\in \mathcal{S}_{t,e}} p_k \bmtheta_{t,e}^k \end{align}

Let $\mathcal{\hat{Q}}$ be a binary set that has length E, and each value is 1 or 0, representing whether at least one client has done the $i^{th}$ epoch. 
So, $Q_t = \sum_{i=1}^E \hat{Q}_i$, where $\hat{Q}_i$ is the $i^{th}$ element of the set. 

where $Q_t$ is the number of non-zero $N_{t,e}$ and $\mathcal{S}_{t,e}$ is the set of devices that finish e local epochs.

\end{proof}

% -------------------------------------------------------------------------

\end{document}